\documentclass[a4paper,11pt]{article}
\usepackage[dvips]{graphicx,color}
\usepackage{amssymb}
\usepackage{amsthm}
\usepackage{amsmath}
\usepackage[ruled,vlined,linesnumbered]{algorithm2e}
\usepackage{tikz}
\usepackage{pgfkeys}
\usepackage{pgfplots}

\theoremstyle{definition}
\newtheorem{theorem}{Result}
\theoremstyle{definition}
\newtheorem{conjecture}{Conjecture}
\begin{document}
\title{The Eigenvalues Entropy as a Classifier Evaluation Measure}
\author{Doulaye Demb\'el\'e}
\date{Institut de G\'en\'etique et de Biologie Moleculaire et
  Cellulaire\\
  CNRS UMR7104, INSERM U964 and Universit\'e de Strasbourg\\
  67400 Illkirch, France~}
\maketitle
\begin{abstract}
Classification is a machine learning method used in many
practical applications: text mining, handwritten character recognition,
face recognition, pattern classification, scene labeling, computer vision,
natural langage processing.
A classifier prediction results and training set information
are often used to get a contingency table which is used to quantify the
method quality through an evaluation measure. Such measure, 
typically a numerical value, allows to choose a suitable method among 
several. Many evaluation measures available in the literature are less
accurate for a dataset with imbalanced classes.
In this paper, the eigenvalues entropy is used as an evaluation measure
for a binary or a multi-class problem. 
For a binary problem, relations are given between the 
eigenvalues and some commonly used  
measures, the sensitivity, the specificity, the area under the operating
receiver characteristic curve and the Gini index.
A by-product result of this paper is an estimate
of the confusion matrix to deal with the curse of the imbalanced classes.
Various data examples are used to show the better performance of the
proposed evaluation measure over
the gold standard measures available in the literature.
\end{abstract}


\section{Introduction}
An evaluation measure allows to quantify the quality of a
classification method, 
that allows to choose one method among many. A classifier is used
to predict the class information of a dataset observations. 
The prediction result and the training set true class information
are then used 
to get a contingency table or a confusion matrix (CM). The diagonal entries 
of this matrix record the correctly predicted observations while the 
off-diagonal entries record those wrongly predicted. A similar matrix 
is obtained when comparing the rating results of two experts or when 
comparing two clustering methods results. 
From two raters results, the class (category) index is associated with each 
observation, 
that leads to an integer-valued CM. A classifier or a 
clustering method consists of a hard or a soft assignment of the observations 
to the classes. The hard assignment is similar to an expert rating and then 
leads to an integer-valued CM. In the soft assignment, membership 
values are associated with each observation to show its proximity with 
the classes/clusters. The membership values associated with an observation 
are very often normalized to have sum one. The index of the maximum 
membership value for each observation can be used to get 
an integer-valued CM. Another solution consists of a direct 
use of the membership values to get a real-valued CM. 
For a good classifier or for high agreement between two experts, two 
clustering methods using the same number of clusters, the diagonal entries 
of the CM are larger than the off diagonal entries.

Many evaluation measures have been proposed in the literature: the 
sensitivity (recall), the specificity (inverse recall), the precision 
(confidence), the accuracy (Rand index), the Jaccard index, the F1-score, 
the receiver operating characteristic (ROC), the area under the 
ROC curve (AUC), Matthews' correlation coefficient or Pearson correlation, 
\ldots, \cite{Cramer-1946,Cohen-1960,Matthews-1975,Fowlkes-al-1983,
Jain-al-1988,Baldi-al-2000,Gorodkin-2004,Ferri-al-2009,Flach-2019}.
The binary classification problem received more attention and a multi-class 
problem is sometime converted into one or many binary problems. In the 
conversion method leading to a single binary problem, the pairs counting of
the observations that are (are not) assigned into the same class according 
to their true and prediction status are used,
\cite{Jain-al-1988,Albatineh-al-2006,Lei-al-2017,Gosgens-al-2021}. 
In the second conversion approach, each class-versus the rest analyses are 
done, and an average of the obtained measures is often reported.
For the multi-class 
performance score method in \cite{Kautz-al-2017}, the sensitivity, 
the specificity and the accuracy measures are extended through a
Bayes statistical 
test involving a binomial distribution.
There are some multi-class problem evaluation measures based on a direct
use of the confusion matrix entries:
the Cramer's correlation coefficient, \cite{Cramer-1946}, the
Matthews correlation coefficient \cite{Gorodkin-2004} and the Cohen's
kappa, \cite{Cohen-1960}. 
Entropy, conditional entropy and mutual entropy obtained from the CM entries
have been used as evaluation measures in 
\cite{Shannon-1948,Kononenko-al-1991,Cover-al-2006,Vinh-al-2010}.
The entropy of the misclassified observations or the confusion 
entropy (CEN), is computed by the method proposed in \cite{Wei-al-2010},
see also \cite{Delgado-al-2019}. 
Some desirable properties for an evaluation measure are provided
in \cite{Gosgens-al-2021}.

An evaluation measure is typically a numerical value which varies in an 
interval. This numerical value may be sensitive to the imbalanced 
ratio of the classes, \cite{Feinstein-al-1990,Hand-2009,
Andres-al-2004,Luque-al-2019,Powers-2020}. 
The imbalanced ratio quantifies how is the difference 
between the largest and the smallest class size,
$ir=(smallest\ size)/ (largest\ size)$. For a balanced classification 
problem, the imbalanced ratio goes towards $1$. This condition is not 
fulfilled for an imbalanced classification problem where the $ir$ can be 
very small, close to zero. When we compare the prediction results for one 
class against the rest, an imbalanced classification very often occurs.
In the method proposed, 
the confusion matrix entries are first adjusted
for each class and an eigenvalue decomposition method is used to obtain
an evaluation measure which especially varies with the off-diagonal entries
of the confusion matrix. 
The  proposed measure value is one when all off-diagonal entries of the CM
are zero. Then, its value decreases continuously towards zero when the
off-diagonal entries increase to be the same as the diagonal entries
for all classes independently of their sizes (bad prediction results). 
The eigenvalues are computed from a symmetric matrix.
Lower and upper bounds for these eigenvalues are obtained from the 
transformed confusion matrix entries.
Another result presented is an estimate of the CM which can be 
directly used
to deal with the curse of imbalanced classification problem.
The contributions in this paper are given as follows:
\begin{itemize}
\item A new classifier evaluation measure based on eigenvalues entropy is 
	proposed.
\item A lower and an upper bounds for the eigenvalues are provided using
	the confusion matrix entries.
\item For a binary problem, relations between the eigenvalues and the AUC, 
	the sensitivity, the specificity and the Gini index are shown.
\item An estimate of the confusion matrix is provided. This estimate
        leads to an 
	improvement of the measures sensitive to imbalanced classes.
\end{itemize}
In the next section, the notations and some evaluation measures
commonly used are presented. The method proposed and some theoretical
results are presented in section \ref{sect-method}. In section 
\ref{sect-results},
some comparison results are shown and the conclusions appear
in section \ref{sect-conclusion}.

\section{Notations and related work}
We consider a training set containing $m$
observations and $n$ classes, $n=2$ for a binary problem and
$n>2$ for a multi-class problem. 
The classifier prediction results and the true class information are 
used to obtain a contingency table or a confusion matrix (CM) denoted 
$\mathbf{M}$, where the entry $m_{ij}$ is the number of observations from 
class $i$ predicted for class $j$. 
The entries $m_{ii}$ are the correct prediction results while the 
entries $m_{ij}$ where $i\neq j$ correspond to the incorrect prediction.
\[\mathbf{M} = \left(\begin{array}{cccc}
m_{11} &m_{12} &\ldots &m_{1n}\\
m_{21} &m_{22} &\ldots &m_{2n}\\
\vdots &\vdots &\ddots &\vdots\\
m_{n1} &m_{n2} &\ldots &m_{nn}
\end{array}\right)\]
The sum of the entries in the row $i$ and the column $j$ of the matrix
$\mathbf{M}$ are denoted $m_{i.}$ and $m_{.j}$, respectively,
$m_{.j}$ is the size of the $j$-th class.
The entry $m_{ij}$ is usually an integer and can
be real-valued when a soft classifier is used. For the contingency table
associated with a binary problem, like a medical diagnostic test, the
four entries are the true positives, $m_{11}$, the false positives 
$m_{12}$, the false negatives, $m_{21}$, and the true negatives, 
$m_{22}$. A zero diagonal entry is not usually expected and 
$\mathbf{M}$ is very often a non-symmetric matrix.

The expressions of some commonly used measures are provided
in relations (\ref{eq-sen}) - (\ref{eq-auc}), for a binary problem.
\begin{eqnarray}
\mbox{recall/sensitivity:} &sen& = \frac{m_{11}}{m_{11} + m_{21}} 
    =\frac{1}{1+\frac{m_{21}}{m_{11}}} \label{eq-sen}\\
\mbox{inverse recall/specificity:} &spe& = \frac{m_{22}}{m_{22} + m_{12}} 
    =\frac{1}{1+\frac{m_{12}}{m_{22}}} \label{eq-spe}\\
\mbox{precision/confidence:} &pre& = \frac{m_{11}}{m_{11} + m_{12}} 
    =\frac{1}{1+\frac{m_{12}}{m_{11}}} \label{eq-pre}\\
\mbox{inverse precision:} & & \frac{m_{22}}{m_{22} + m_{21}} 
    =\frac{1}{1+\frac{m_{21}}{m_{22}}} \\
\mbox{F1-score:} &f1s& = \frac{2\times pre\times sen}{pre+sen}\\
\mbox{Fowlkes \& Mallows index:} &fmi& = (sen\times pre)^{1/2}\\
\mbox{area under ROC curve:} &auc& = (sen + spe)/2 \label{eq-auc}
\end{eqnarray}
The sensitivity (true positive rate), the specificity (true negative rate) 
and the precision (confidence)
depend on the scalars $m_{21}/m_{11}$, $m_{12}/m_{11}$ and $m_{12}/m_{22}$.
When all off-diagonal entries are zero, then these measures are all
ones. Theoretically, the sensitivity, the specificity and the precision 
vary between $0$ and $1$. 
The plot of the true positive rate versus the false positive rate, 
$m_{12}/(m_{12}+m_{22})$, is the 
receiver operating characteristic (ROC), 
\cite{Furnkranz-al-2005,Hand-2009}. The area under the ROC curve
(AUC) is an evaluation measure which expression is given in relation
(\ref{eq-auc}) for a binary problem, \cite{Powers-2020}. 
An extension of the AUC to a 
multi-class problem is provided in \cite{Hand-al-2001}.

The expressions for other measures are provided in relations 
(\ref{eq-acc}) - (\ref{eq-mcc}). These measures can be used
for a binary or a multi-class problem.
\begin{eqnarray}
\mbox{accuracy:} &acc& = \frac{1}{m}\sum_{i=1}^nm_{ii}  \label{eq-acc} \\
\mbox{Cohen kappa:} &\kappa& = 
    \frac{m\sum_{i=1}^nm_{ii}-\sum_{i=1}^nm_{i.}m_{.i}} 
     {m^2-\sum_{i=1}^nm_{i.}m_{.i}} \label{eq-kappa}\\
\mbox{Matthews' correlation coef.:} &mcc& = 
    \frac{m\sum_{i=1}^nm_{ii}-\sum_{i=1}^nm_{i.}m_{.i}}
      {\sqrt{(m^2-\sum_{i=1}^nm_{i.}^2)(m^2-\sum_{i=1}^nm_{.i}^2)}}
     \label{eq-mcc}
\end{eqnarray}
The accuracy is a well known measure which is widely used in the
literature. A Bayesian framework is used in 
\cite{Brodersen-al-2010,Brodersen-al-2012} to have posterior distribution
and posterior accuracy, a balanced accuracy.
The Cohen kappa measure received a large attention in the literature
\cite{Brennan-al-1981,Feinstein-al-1990}. However, it does not perform
well for asymmetric marginal distribution, \cite{Andres-al-2004}.
The Matthews correlation coefficient is also widely used in the 
literature, 
\cite{Baldi-al-2000,Boughorbel-al-2017,Chicco-al-2021,Foody-2023}, 
but it is sensitive to imbalanced classes problem.

Other evaluation measures are based on relative values of the
CM entries. The row and column sums are divided by a total number 
of observations to get probabilities, $p_i$,
which are used to obtain an entropy: $-\sum_{i=1}^np_i\log(p_i)$,
\cite{Kononenko-al-1991}. The mutual and the
conditional entropies are also computed for the CM rows and columns. 
The entropy, $H(pred.,class)$, and the mutual entropy (mutual information),
$I(pred.,class)$ associated with the prediction
assignments and the true class memberships are obtained using the following
relations:
\begin{eqnarray}
H(pred.,class) &=& -\sum_{i=1}^n\sum_{j=1}^n\frac{m_{ij}}{m}
	\log\frac{m_{ij}}{m}\\
I(pred.,class) &=& \sum_{i=1}^n\sum_{j=1}^n\frac{m_{ij}}{m}
	\log\frac{m_{ij}/m}{m_{i.}m_{.j}/m^2}
\end{eqnarray}
The normalized mutual information (NMI) is defined as
$I(pred.,class)/H(pred.,class)$, \cite{Vinh-al-2010}. 
The  entropy of the misclassified observations, the confusion entropy (CEN), is
computed by the method proposed in \cite{Wei-al-2010}. The CEN value 
is the overall misclassification of the $m$ observations and is obtained from 
two intermediary matrices. The entries of the first intermediary matrix 
are the misclassification
probabilities of class $i$ observations to class $j$ subject to class $j$. The
entries of the second intermediary matrix are the misclassification 
probabilities of the class $i$ observations to class $j$ subject to class 
$i$. The diagonal entries of the intermediary matrices
are all zeros and the $n(n-1)$ entries are used to get a confusion matrix for
each class and then a measure, $CEN_j$, $j=1,2,\ldots,n$. On the other hand,
the CM entries are used to get a probability, $P_j$, for each class. The 
entropies $CEN_j$ and the class probabilities $P_j$ are combined to obtain 
the CEN evaluation measure which typically varies between 0 and 1.
For a binary classification, the CEN value can exceed $1$.
The CEN measure method has been modified in \cite{Delgado-al-2019},
MCEN, to have a value which always varies in $[0,1]$. 

The expressions for the
conversion of a multi-class CM into a single binary CM are recalled here.
For a dataset with 
$m$ observations, there are $M=m(m-1)/2$ possible pairs,
e.g. for $m=4$ we have a set of $6$ pairs: \{(1,2), (1,3), (1,4),
(2,3), (2,4), (3,4)\}. Let $a=m_{11}$, $b=m_{12}$, $c=m_{21}$ and
$d=m_{22}$ be the conversion results, 
the following relations allow to get the entries
of a $2\times 2$ confusion matrix, \cite{Jain-al-1988,Albatineh-al-2006}.
\begin{eqnarray}
a = \frac{1}{2}\sum_{i=1}^n\sum_{j=1}^nm_{ij}^2 - \frac{m}{2} &\mbox{ ; }&
b = \frac{1}{2}P - a\label{eq-multi2two}\\
c = \frac{1}{2}Q - a &\mbox{ ; }&
d = M-\frac{1}{2}(P+Q) + a\label{eq-multi2two2}
\end{eqnarray}
where $P=\sum_{i=1}^nm_{i.}^2$ and $Q=\sum_{j=1}^nm_{.j}^2$.
 
 The starting point of the proposed measure is the confusion matrix 
 $\mathbf{M}$ which is assumed to be available.
A column stochastic matrix $\mathbf{P}$ is computed as follows.
\begin{eqnarray}
    \mathbf{P} &=& \mathbf{MD}^{-1} = \mathbf{D}^{-1/2}
    \tilde{\mathbf{M}}\mathbf{D}^{-1/2}\label{eq-mat-p0}\\
     &=& \left(\begin{array}{ccccc}
    p_{11} &p_{12} &p_{13} &\ldots &p_{1n}\\
    p_{21} &p_{22} &p_{23} &\ldots &p_{2n}\\
    p_{31} &p_{32} &p_{33} &\ldots &p_{3n}\\
    \vdots &\vdots &\vdots &\ddots &\vdots\\
    p_{n1} &p_{n2} &p_{n3} &\ldots &p_{nn}\\
    \end{array} \right) \label{eq-mat-p}
\end{eqnarray}   
where $p_{ij}=m_{ij}/m_{.j}$, $\sum_{i=1}^np_{ij}=1$ and $\mathbf{D}^{-1}
=diag(1/m_{.1},1/m_{.2},\ldots,1/m_{.n})$ 
is a diagonal matrix and $\tilde{\mathbf{M}}$ is a modification of the
confusion matrix:
\begin{equation}
    \tilde{\mathbf{M}} = \mathbf{D}^{1/2}\mathbf{M}\mathbf{D}^{-1/2}
    \label{eq-mmatt}
\end{equation}
$\tilde{m}_{ij} = m_{ij}/\sqrt{m_{.j}/m_{.i}}$.
It is assumed that all column sums $m_{.j}$, are nonzero, i.e., there 
is no empty class. Otherwise, $1/n$ can be added to each entry of 
$\mathbf{M}$ before processing. The off-diagonal entries of $\mathbf{M}$ 
are modified to get $\tilde{\mathbf{M}}$. This modification aims to
make a balance 
between the classes sizes for improving the imbalanced ratio by moving it 
away from zero. Some properties of $\tilde{\mathbf{M}}$, including its direct
use for computing an evaluation measure, are discussed later.

The matrix $\mathbf{P}$ has been used in \cite{Hay-1988} 
as an error probability matrix for prediction and in 
\cite{Theissler-al-2022} for comparing $k$ classifiers results
by applying the multidimensional scaling method on a
$\ell_1$-norm matrix. Ideally, $\mathbf{P}$ is an identity matrix to 
indicate a correct assignment of the observations in their actual classes.
The ideal classifier is unlikely because of uncertainty for some 
observations, classes overlap and the presence of noise. $\mathbf{P}$
can be viewed as a directed weighted graph matrix and its transpose
is the transition matrix of a $n$-state Markov chain (MC). When all off 
diagonal entries of $\mathbf{P}$ are zero, there is no relation 
between the MC process states, an ideal classifier. 
For the random walk process associated with a MC transition matrix, 
one is very often interested in the long run behavior, the stationary
distribution vector. However, this vector does not allow to highlight the 
amount of the off diagonal entries. For a class-overlap 
problem or a non-perfect classifier, some off diagonal entries of 
$\mathbf{P}$ are non null. The classification measure proposed aims 
to quantify the amount of the non-zero off diagonal entries associated 
with $\mathbf{M}$. Since, the eigenvalues of a matrix vary continuously
with its entries, \cite[page 497]{Meyer-2000}, the measure proposed
is based on the eigenvalues of a 
symmetric matrix $\mathbf{B}$ associated with $\mathbf{P}$.

\section{Method}\label{sect-method}
The idea of the proposed evaluation measure is the following: the amount 
of the off diagonal entries get larger for a bad classifier. 
The eigenvalues of an ideal matrix
$\mathbf{P}$ are all ones. An increase of the off diagonal entries of
$\mathbf{M}$ will lead to a decrease (in module) for some 
eigenvalues of the matrix $\mathbf{P}$. This matrix is not necessary symmetric
and can have complex-valued eigenvalues. 
To have only real-valued eigenvalues, the following symmetric matrix is used.
\begin{equation}
\mathbf{B} = (\mathbf{P} + \mathbf{P}^{\mathsf{T}})/2 
 = \mathbf{D}^{-1/2}\left\{(\tilde{\mathbf{M}} +
 \tilde{\mathbf{M}}^{\mathsf{T}})/2\right\}\mathbf{D}^{-1/2} \label{eq-mat-b1}
\end{equation}
where $b_{ij} = \left[m_{ij}/m_{.j} + m_{ji}/m_{.i}\right]/2 = 
    \left(p_{ij}+p_{ji}\right)/2 \mbox{ ; }i,j=1,2,\ldots,n$

If $\mathbf{P}$ is symmetric then $\mathbf{B}=\mathbf{P}$.
The entries of the matrix $\mathbf{M}$ are scaled for 
each class to obtain the matrix $\mathbf{P}$. That is useful for an 
imbalanced classification problem where a direct use of the entries $m_{ij}$ 
may have a negative influence on some evaluation measures.

\subsection{Eigenvalues entropy as a classifier evaluation measure} 
The spectrum of the matrix $\mathbf{B}$ is the set of its eigenvalues
counted with multiplicity: 
$\sigma(\mathbf{B}) = \{\lambda_1,\lambda_2,\ldots, \lambda_n\}$,
$\lambda_i\geq \lambda_{i+1}$, $i=1,2,\ldots,n-1$. Only positive-valued
eigenvalues are assumed to contribute positively to a classifier quality.
The standardized eigenvalues entropy (EVE) is used as a classifier evaluation
measure.
\begin{equation}
eve(\mathbf{B}) = -\sum_{i=1}^n\eta_i\log(\eta_i)/\log(n)
    = -\sum_{i=1}^n\eta_i\log_n(\eta_i)\label{eq-eve}
\end{equation}
where $\eta_i=\lambda_i/n^*$,
$\lambda_i$ is a positive-valued eigenvalue of $\mathbf{B}$ and
$n^*$ is the sum of the positive-valued eigenvalues.
This measure value varies in the interval $[0,1]$.
For an ideal classifier, 
$\mathbf{B} = \mathbf{I}_n$ ($n$-order identity matrix), $\sigma(\mathbf{B})
=\{1,1,\ldots,1\}$, $\eta_i=1/n$, $\forall i$, and
$eve(\mathbf{B})=1$. On the other hand, when a 
classifier assigns equally the same observations from each class to all 
classes, then $\mathbf{B} = (1/n)\mathbf{J}_n$
(an $n$-order all-ones matrix), leading to
$\sigma(\mathbf{B}) =\{1,0,\ldots,0\}$ and $eve(\mathbf{B})=0$. 
A near one EVE value is associated with a high quality classifier, while a 
near zero value is associated with a bad quality classifier.

A direct method allowing to get the eigenvalues $\lambda_i$ consists of 
solving the $n$-order polynomial equation: 
$p(\lambda) = det(\mathbf{B} - \lambda\mathbf{I}_n) = 0$.
The continuous variation of the eigenvalues with the CM entries come
from their dependence on the coefficients of the polynomial solved
to get them.
There are many tools for computing the eigenvalues of a symmetric matrix,
\cite[Chapter 8]{Golub-al-1996}. We are especially interested in the
variation of the eigenvalues of $\mathbf{B}$ as a function of the
matrix $\mathbf{M}$ entries. The trace of the matrix $\mathbf{B}$ for an
ideal classifier is $n$. This value decreases with nonzero off-diagonal 
entries in $\mathbf{P}$. Let us assume that for a ``good'' classifier,
more than half of the observations in each class are correctly predicted, 
i.e.:
\begin{equation}
    p_{jj}  > \sum_{i=1,i\neq j}^np_{ij} \mbox{ ; } \forall j=1,2\ldots, n
    \label{eq-diagonal-dominant}
\end{equation}
The above relation means the off-diagonal entries sum for each column
of $\mathbf{P}$ is less than the diagonal entry, i.e. $\mathbf{P}$
is a diagonally dominant matrix. This criterion matches a common intuition,
otherwise, one can suspect a lot of incorrect predictions, 
a heavy overlap in classes or uncertainty in the training set labels. 
We have the following results.

\begin{theorem}
The matrices $\mathbf{M}$ and $\tilde{\mathbf{M}}$ have the same eigenvalues.
\end{theorem}           
\begin{proof}
Using $\tilde{\mathbf{M}}\mathbf{u} = \lambda\mathbf{u}$ $\Leftrightarrow$
$\mathbf{D}^{1/2}\mathbf{MD}^{-1/2}\mathbf{u} = \lambda\mathbf{u}$, that allows 
to write $\mathbf{MD}^{-1/2}\mathbf{u} = \lambda\mathbf{D}^{-1/2}\mathbf{u}$
or $\mathbf{Mv} = \lambda\mathbf{v}$. That means only the eigenvectors
differ in the 
eigenvalue decomposition of the matrices $\mathbf{M}$ and 
$\tilde{\mathbf{M}}$. 
\end{proof}

\begin{theorem}
The matrices $\mathbf{P}$ and $\mathbf{B}$ have the same trace.
\end{theorem}
\begin{proof}
The trace of a matrix is the sum of its eigenvalues, on the one hand. 
On the other hand, a matrix and its transpose have the same eigenvalues,
$trace(\mathbf{P}) = \sum_{i=1}^n\sigma(\mathbf{P})=
0.5\sum_{i=1}^n\left(\sigma(\mathbf{P}) + \sigma(\mathbf{P}^{\mathsf{T}})\right)
=trace(\mathbf{B})$.
\end{proof}

\begin{theorem}
When relation (\ref{eq-diagonal-dominant}) is fulfilled, then 
$\mathbf{P}$ and $\mathbf{B}$ have only positive-valued eigenvalues.
\end{theorem}         
\begin{proof}
Using $\sigma(\mathbf{P}) = \sigma(\mathbf{P}^{\mathsf{T}})$, 
relation (\ref{eq-diagonal-dominant}) means $\mathbf{P}$ is
a diagonally dominant matrix and then is nonsingular with 
positive-valued eigenvalues or a positive-definite matrix: 
$\mathbf{x}^{\mathsf{T}}\mathbf{P}\mathbf{x}>0$, where $\mathbf{x}$
is a nonzero $n$-order vector.
Since $\mathbf{x}^{\mathsf{T}}\mathbf{B}\mathbf{x}=
\mathbf{x}^{\mathbf{T}}\left[\left(\mathbf{P} + 
\mathbf{P}^{\mathsf{T}}\right)/2\right]\mathbf{x} = 
\mathbf{x}^{\mathsf{T}}\mathbf{P}\mathbf{x}$, then the matrix 
$\mathbf{B}$ is also positive definite for a ``good'' classifier.
\end{proof}

\begin{theorem}
The matrices $\mathbf{M}$ and $\mathbf{P}$
have the same number of non-zero eigenvalues.
\end{theorem}        
\begin{proof}
From relation (\ref{eq-mat-p0}), the matrix $\mathbf{P}$ is derived from
$\tilde{\mathbf{M}}$. Then, these two matrices are equivalent
and have the same 
rank, \cite[page 137]{Meyer-2000}. Since the rank of a matrix is the number
of its non-zero eigenvalues, $\tilde{\mathbf{M}}$ and
$\mathbf{P}$ have the same number of non-zero eigenvalues.
The proof is completed by using Result 1.
\end{proof}

The matrix $\tilde{\mathbf{M}}$ is an estimate of the 
confusion matrix. However, the sum of the entries of this
estimate, $\tilde{m}$, may differ to the total number $m$ of observations.
Instead of using $\mathbf{M}$, the matrix 
$\tilde{\mathbf{M}}$ can be used for improving an imbalanced
ratio dependent evaluation measure.
From relations (\ref{eq-sen}) - (\ref{eq-acc}), the sensitivity, the 
specificity and the AUC can be expressed using 
the the entries of the matrix $\mathbf{P}$:
$sen=p_{11}$, $spe=p_{22}$, $auc=(p_{11} + p_{22})/2$
$=1-(p_{12}+p_{21})/2$, $p_{12}$ and $p_{21}$ are the false positive
and the false negative rates, respectively. 
When the estimate confusion matrix $\tilde{\mathbf{M}}$ is used, the 
expressions for the sensitivity, the specificity and the precision 
become:
\begin{eqnarray}
sen &=& \frac{m_{11}}{m_{11} + m_{21}\sqrt{\frac{m_{.2}}{m_{.1}}}} 
    = \frac{1}{1+\frac{m_{21}}{m_{11}}\sqrt{\frac{m_{.2}}{m_{.1}}}}
    \label{eq-sen2}\\
spe &=& \frac{m_{22}}{m_{22} + m_{12}\sqrt{\frac{m_{.2}}{m_{.1}}}} 
= \frac{1}{1+\frac{m_{12}}{m_{22}}\sqrt{\frac{m_{.1}}{m_{.2}}}}
\label{eq-spe2}\\
pre &=& \frac{m_{11}}{m_{11} + m_{12}\sqrt{\frac{m_{.1}}{m_{.2}}}} 
    = \frac{1}{1+\frac{m_{12}}{m_{11}}\sqrt{\frac{m_{.1}}{m_{.2}}}}
    \label{eq-pre2}
\end{eqnarray}

The difference between the above expressions and those in relations 
(\ref{eq-sen}) - (\ref{eq-pre}) is the square root of the imbalanced
ratio: $\sqrt{m_{.1}/m_{.2}}$ or its inverse. The contribution of the
off-diagonal entries in the calculation of the recall, the inverse
recall and the precision then depend on the imbalanced ratio. 
When both classes have the same size, $m_{.1}=m_{.2}$, then there is no 
difference between the above expressions based on $\tilde{\mathbf{M}}$ and 
those based on $\mathbf{M}$, $\tilde{\mathbf{M}}=\mathbf{M}$.
Consider a binary problem where the size of 
the positive outcomes is smaller than that of the negative outcomes, 
$ir=1/5$, and where the off-diagonal entries are a quarter of the 
corresponding diagonal entry: $m_{21}=m_{11}/4$ and $m_{12}=m_{22}/4$.
For this example, the sensitivity, the specificity, the precision and the
area under the ROC curve are, respectively, $0.8$, $0.8$, $0.44$ and 
$0.8$ when using the matrix $\mathbf{M}$ and $0.64$, $0.9$, $0.64$ and 
$0.77$, when using the matrix $\tilde{\mathbf{M}}$. In this example, 
the use of the estimate confusion matrix leads to an increase of the 
specificity ($12.5\%$) and the precision ($45.45\%$), but a decrease of the 
sensitivity ($20\%$) and the area under the ROC curve ($3.75\%$).

\subsection{Bounds for the eigenvalues of the matrix $\mathbf{B}$}
The objective here is to study the behavior of the eigenvalues of the matrix 
$\mathbf{B}$ as a function of the matrix $\mathbf{P}$ entries.
When the CM matrix is symmetric, then $\mathbf{B}=\mathbf{P}$ is also 
a stochastic symmetric matrix. Therefore, from the works of Perron and 
Frobenius, \cite{Perron-1907,Frobenius-1912},
the maximum eigenvalue is $\lambda_1=1$ and the modulus of the other
eigenvalues is less than one. The confusion matrix is not necessary
symmetric. However, the matrix $\mathbf{B}$ is nonnegative, $b_{ij}\geq 0$. 
Let $\mathbf{q}=(b_{11},b_{22},\ldots,b_{nn})^{\mathsf{T}}$ be a vector which
components are the diagonal entries of the matrix $\mathbf{B}$. This vector
entries are used as the components of a diagonal matrix $\mathbf{Q}$ which
allows to define another matrix:
\begin{equation}
    \mathbf{A} = \mathbf{Q}^{-1/2}\mathbf{BQ}^{-1/2}
    = \mathbf{I}_n + \mathbf{R}\label{eq-mat-a}
\end{equation}
where 
\[a_{ij} = \frac{1}{\sqrt{b_{ii}b_{jj}}}b_{ij}
=\frac{1}{2\sqrt{p_{ii}p_{jj}}}(p_{ij} + p_{ji})\]
The diagonal entries of the symmetric
matrix $\mathbf{A}$ are all one and $\mathbf{R}$ is a $n$-order 
matrix with zero as diagonal entries and the off-diagonal entries of 
the matrix $\mathbf{A}$. 
It is assumed that no diagonal entry of the CM is zero. The CM can be 
modified using $\mathbf{M}+(1/n)\mathbf{J}_n$ to fulfill this condition.
The inertia of a symmetric matrix is a triplet of nonnegative 
integers recording the number of the positive, zero and negative 
eigenvalues. From the Sylvester's law of inertia, the matrices 
$\mathbf{A}$ and $\mathbf{B}$ have the same inertia, 
\cite[page 563]{Meyer-2000}. 
Since all entries of the matrix $\mathbf{P}$ verify $p_{ij}\leq 1$, 
the first term on the right
hand side of the above relation shows that $b_{ij}\leq a_{ij}$, i.e.
$\mathbf{B}\leq \mathbf{A}$. The Gershgorin circles theorem, 
\cite{Varga-2004}, is applied to the matrix $\mathbf{A}$ to have 
the bounds for the eigenvalues of $\mathbf{B}$.

The spectrum of the matrix $\mathbf{A}$ is 
$\sigma(\mathbf{A} ) = \{\mu_1,\mu_2,\ldots,\mu_n\}$, $\mu_i\geq \mu_{i+1}$,
$i=1,2,\ldots,n-1$, and the entries of its row-sums vector 
$\mathbf{r}^{\mathbf{A}}$ are ($i=1,2,\ldots,n$):
\begin{equation}
r_i^{\mathbf{A}} = 1+\sum_{j=1;j\neq i}^n\frac{1}{2\sqrt{p_{ii}p_{jj}}}
(p_{ij}+p_{ji}) = 1+r_i \label{eq-mat-a-rip}
\end{equation}
where $r_i\equiv r_i^{\mathbf{R}}$ is the $i$-th entry of the matrix $\mathbf{R}$ 
row-sums vector.
\begin{small} 
\begin{figure}[!ht]
\begin{center}
\includegraphics[width=0.3\textwidth]{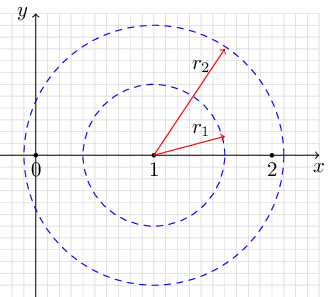}
\end{center}
\caption{Gershgorin circles}
\label{fig-gershgorin-circles}
\end{figure}
\end{small}

Figure \ref{fig-gershgorin-circles} shows the possible Gershgorin circles
for the matrix $\mathbf{A}$ where the off-diagonal entries sum may be less 
($r_1$) or greater ($r_2$) than $1$. From the Gershgorin circles theorem, 
\cite{Varga-2004},
all circles containing the eigenvalues of the matrix $\mathbf{A}$ have the 
same center, $1$, and these circles' radii depend on the off diagonal entries 
of $\mathbf{P}$, see relation (\ref{eq-mat-a-rip}). Thank to symmetry,
the eigenvalues of the matrices $\mathbf{B}$ and $\mathbf{A}$ are localized 
on the $x$-axis of this Figure. Therefore, the lower and upper bounds of the
matrix $\mathbf{B}$ eigenvalues are $thr_{min}=1-r$ and $thr_{max}=1+r$, 
respectively, where $r=\max(r_i)$.

Using the nonnegative property of the matrices $\mathbf{A}$ and $\mathbf{B}$,
which are assumed to be irreducible, their Perron roots allow to write:
$\mathbf{Ax} = \rho(\mathbf{A})\mathbf{x}$ and 
$\mathbf{By} = \rho(\mathbf{B})\mathbf{y}$, 
where $\mathbf{x}$ and $\mathbf{y}$ are positive-valued entries vectors.
Thank to symmetry, 
$\mathbf{y}^{\mathsf{T}}\mathbf{B} = \rho(\mathbf{B})\mathbf{y}^{\mathbf{T}}$.
These relations allow to write
$\mathbf{y}^{\mathsf{T}}\mathbf{Ax} = \rho(\mathbf{A})
\mathbf{y}^{\mathsf{T}}\mathbf{x}$ and 
$\mathbf{y}^{\mathsf{T}}\mathbf{Bx} = \rho(\mathbf{B})
\mathbf{y}^{\mathsf{T}}\mathbf{x}$. 
Setting $\mathbf{A} = \mathbf{B} + \mathbf{E}$, where $\mathbf{E}\geq 0$,
then $\rho(\mathbf{A}) = \rho(\mathbf{B}) + 
\mathbf{y}^{\mathsf{T}}\mathbf{Ex}/\mathbf{y}^{\mathsf{T}}\mathbf{x}$.
Therefore $\lambda_1=\rho(\mathbf{B})\leq \mu_1=\rho(\mathbf{A})$.
The Perron root of $\mathbf{B}$ is less or equal to that of $\mathbf{A}$.
The relation between all the eigenvalues of the matrices $\mathbf{B}$
and $\mathbf{A}$ is formulated in the following conjecture.
\begin{conjecture}
For the $n$-order nonnegative symmetric matrices $\mathbf{B}$ and 
$\mathbf{A}$ where $\mathbf{B}\leq \mathbf{A}$, 
$\sigma(\mathbf{B}) =\{\lambda_1,\lambda_2,\ldots,\lambda_n\}$ and 
$\sigma(\mathbf{A}) =\{\mu_1,\mu_2,\ldots,\mu_n\}$, then
\[|\lambda_i|\leq |\mu_i|\mbox{ ; }i=1,2,\ldots,n\]
\end{conjecture}

\subsection{Two-class problem}
For a binary problem and using relation (\ref{eq-mat-a-rip}), 
there is a unique $r$ value which expression is
$r=(p_{12}+p_{21})/2\sqrt{p_{11}p_{22}}$ and the eigenvalues of $\mathbf{A}$ 
are $\mu_1=1+r$ and $\mu_2=1-r$.

The evaluation measure in relation (\ref{eq-eve}) is for any classifier,
binary or multi-class. Explicit expressions are obtained for the eigenvalues 
of a binary problem. In this case, the matrices $\mathbf{P}$ and 
$\mathbf{B}$ are given as follows:
\begin{equation}
    \mathbf{P} =\left(\begin{array}{cc}
    p_{11} &p_{12}\\  p_{21} &p_{22}
    \end{array}\right) \mbox{ ; }
    \mathbf{B} =\left(\begin{array}{cc}
    p_{11} &(p_{12}+p_{21})/2\\
    (p_{12}+p_{21})/2 &p_{22}\end{array}\right)\label{eq-mat-p22-b22}
\end{equation}
The expressions for the two eigenvalues of $\mathbf{B}$ are:
\begin{eqnarray}
\lambda_1 &=& 0.5\left(p_{11}+p_{22} + 
    \sqrt{(p_{11}-p_{22})^2 + (p_{12}+p_{21})^2}\right) \label{eq-lambda1}\\
\lambda_2 &=& 0.5\left(p_{11}+p_{22} - 
    \sqrt{(p_{11}-p_{22})^2 + (p_{12}+p_{21})^2}\right) \label{eq-lambda2}
\end{eqnarray}
The eigenvalues are expressed using all the matrix $\mathbf{P}$ entries,
the true positive, the true negative, the false positive and the false 
negative rates. The first eigenvalue is always positive.
When $p_{11}=p_{22}=p$, then $\lambda_1=1$ and $\lambda_2=2p-1$
(we used $p_{12}=1-p_{22}$ and $p_{21}=1-p_{11}$). The second
eigenvalue will be null if $p=0.5$ and negative if $p<0.5$. 
When both $p_{11}$ and $p_{22}$ are greater than $0.5$, then $\lambda_1$ and 
$\lambda_2$ are positive-valued. Positive-valued eigenvalues are also 
obtained for other cases, e.g. $p_{11}=0.7$ and $p_{22}=0.4$. Given a value
for $p_{22}$ and solving relation (\ref{eq-lambda2}) for $\lambda_2=0$ allows
to have: $p_{11} = p_{22}+2-\sqrt{8p_{22}}$.
The (red) curve in figure \ref{fig-lambda2} corresponds to $\lambda_2=0$
and separates positive and negative $\lambda_2$ value locations when
$p_{11}$ and $p_{22}$ vary. A value $p_{jj}<0.5$ corresponds to an
incorrect prediction for more than half of the class $j$ observations.
\begin{small} 
\begin{figure}[!ht]
\begin{center}
\includegraphics[width=0.4\textwidth]{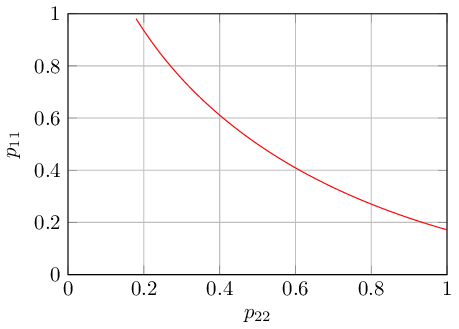} 
\end{center}
\caption{Behavior of $\lambda_2$ when $p_{11}$ and $p_{22}$ vary}
\label{fig-lambda2}
\end{figure}
\end{small}

\subsubsection{Relation between the eigenvalues and some binary problem 
measures}
The following relations show a direct link between the eigenvalues of a
binary problem and the sensitivity, the specificity, the area under the 
ROC curve and the Gini index (G or $gid$).
\begin{eqnarray}
\lambda_1 &=& auc + 0.5\sqrt{2(sen-1)^2 + 2(spe-1)^2}\label{eq-lambda-auc}\\
\lambda_2 &=& auc - 0.5\sqrt{2(sen-1)^2 + 2(spe-1)^2}\label{eq-lambda-auc2}\\
\lambda_1+\lambda_2 &=& 2 auc = gid+1\label{eq-lambda-auc-gini}\\
\lambda_1\lambda_2 &=& 2 auc - 1 - 0.25(sen-spe)^2\label{eq-eigen-prod}\\
    &=& gid - 0.25(sen-spe)^2\label{eq-eigen-prod2}
\end{eqnarray}
Relations (\ref{eq-lambda-auc}) and (\ref{eq-lambda-auc2}) 
show that when the sensitivity 
(the true positive rate) and the specificity (the true negative rate)
are both equal to 1, a perfect classifier, then 
$\lambda_1 = \lambda_2 = auc =1$. When the specificity goes
towards 0, then $\lambda_{1,2} = auc \pm (sen-1)$. A 
similar result is obtained when the sensitivity goes towards zero.
The half sum of the eigenvalues is the area under the ROC curve,
AUC $=(\lambda_1+\lambda_2)/2$. The sum of the eigenvalues is related 
to the Gini index, \cite{Hand-2009}, G = 2AUC - 1 $= \lambda_1+\lambda_2-1$.
On the other hand, 
the product of the eigenvalues is related to the AUC, the
true positive rate and the true negative rate. Since the 
first eigenvalue is always positive, the sign of the eigenvalues product 
depends on the AUC and the true positive and negative rates.
Therefore, the second eigenvalue is non-positive if:
\[ auc \leq 0.5 + 0.125(sen - spe)^2 \mbox{ or } gid\leq 0.25(sen-spe)^2\]
When $sen\approx spe$ in the above relation,
then $auc\approx 0.5$, a bad value. 
That means a non-positive $\lambda_2$ value corresponds to a bad
area under ROC curve and therefore to a worse classifier. The following two 
examples for $p_{11}$ and $p_{22}$ lead to a zero EVE value:
(a) $p_{11}=p_{22}=0.5$ (random assignment) and (b) $p_{11}=0.9$ and
$p_{22}=0.1$ (good and bad prediction results for the first and the second
class, respectively). For both cases, AUC=0.5 and the eigenvalue products 
are $0$ and $-0.16$, respectively. Hence, the EVE value is zero when the
second eigenvalue is negative, independently of the prediction quality
for both classes.

\section{Results and discussion}\label{sect-results}
To show the performance of the proposed measure, synthetic and
real-life binary
and multi-class problem datasets are used. For a binary problem, the results 
for the measures in relations (\ref{eq-sen}) - (\ref{eq-mcc}), 
the normalized mutual information (NMI), the CEN, its 
modification MCEN and the EVE measures are shown. 

The values for most of the measures used in the comparison vary in the 
interval $[0, 1]$. The NMI is computed as the ratio of the mutual entropy 
and the entropy associated with the prediction assignments and the true
class memberships, NMI $\in [0,1]$.
The meaning of CEN is opposite to that of the other measures and the 
expression $cen.s=1-CEN$ is used in the comparison, similarly for MCEN. 
The Matthews
correlation coefficient varies in the intervals $[-1,1]$ and the 
expression $mcc.s=(MCC+1)/2$ is used to have a value in $[0, 1]$. 
Therefore, a near one (zero) measure value corresponds to a good (bad) 
classifier. 

\subsection{Binary examples} 
The following six examples are used.
\begin{equation}
    \mathbf{M}_a = \left(\begin{array}{cc}
        15 &25\\ 15 &25
    \end{array}\right) \mbox{ ; }
    \mathbf{M}_b = \left(\begin{array}{cc}
        45 &5\\ 5 &45
    \end{array}\right) \mbox{ ; }
    \mathbf{M}_c = \left(\begin{array}{cc}
        5 &45\\  45 &5
    \end{array}\right) \nonumber\label{twoclass-exp-abc}
    \end{equation}
\begin{equation}
    \mathbf{M}_1 = \left(\begin{array}{cc}
        125 &30\\   15 &130
    \end{array}\right) \mbox{ ; }
    \mathbf{M}_2 = \left(\begin{array}{cc}
        9 &80\\  1 &210
    \end{array}\right) \mbox{ ; }
    \mathbf{M}_3 = \left(\begin{array}{cc}
        434 &7\\  10 &232
    \end{array}\right) \nonumber\label{eq-balanced}
\end{equation}
The transition matrices associated with $\mathbf{M}_a$ and $\mathbf{M}_b$ 
have the same stationary distribution vector 
$\mathbf{p}=(0.5,0.5)^{\mathsf{T}}$.
For $\mathbf{M}_a$, $50\%$ of the observations are correctly 
classified against $90\%$ for $\mathbf{M}_b$. The matrix
$\mathbf{M}_c$ corresponds to a bad prediction result. 
Near zero measure values are 
expected for matrices $\mathbf{M}_a$ and $\mathbf{M}_c$.
The matrices $\mathbf{M}_1$ and $\mathbf{M}_2$ correspond to a balanced and 
an imbalanced diagnostic tests involving $300$ individuals. The matrix
$\mathbf{M}_3$ corresponds to a Support Vector Machine classification
results for 
the breast cancer dataset, \cite{Mangasarian-al-1990}, which consists of $683$ 
instances, $65\%$ benign and $35\%$ malignant. The imbalanced ratios for 
$\mathbf{M}_2$ and $\mathbf{M}_3$ are $0.034$ and $0.538$, respectively.
The bounds and eigenvalues of the matrices $\mathbf{A}$ and $\mathbf{B}$ 
associated with these six examples are shown in Table \ref{tab-mam3-eigen}.

\begin{table}[!htb]
    \caption{Bounds and eigenvalues of $\mathbf{A}$ and $\mathbf{B}$}
        \label{tab-mam3-eigen}
    \begin{center}
    \begin{tabular}{rcccccc}
      & & &\multicolumn{2}{c}{$\mathbf{A}$}
      &\multicolumn{2}{c}{$\mathbf{B}$}\\
      matrix &$thr_{min}$ &$thr_{max}$ &$\mu_1$ &$\mu_2$
             &$\lambda_1$ &$\lambda_2$\\ \hline
    $\mathbf{M}_a$ &0.000 &2.000 &2.000 &0.000 &1.000 &0.000\\
    $\mathbf{M}_b$ &0.888 &1.111 &1.111 &0.888 &1.000 &0.800\\
    $\mathbf{M}_c$ &-8.00 &10.00 &10.00 &-8.00 &1.000 &-0.80\\
    $\mathbf{M}_1$ &0.827 &1.173 &1.173 &0.827 &1.005 &0.699\\
    $\mathbf{M}_2$ &0.767 &1.233 &1.233 &0.767 &1.019 &0.604\\
    $\mathbf{M}_3$ &0.973 &1.026 &1.026 &0.973 &1.000 &0.948\\ \hline 
    \end{tabular}
\end{center}
\end{table}
For these binary examples, the eigenvalues bounds are tight for a good 
classifier and a large eigenvalues bounds range, $\geq 2$, 
is associated with a bad classifier.

\begin{table}[!htb]
    \caption{Measure results for binary problem examples}
        \label{tab-twoclass-results}
    \begin{center}
    \begin{tabular}{rcccccccc}
     &$\mathbf{M}_a$ &$\mathbf{M}_b$ &$\mathbf{M}_c$ 
        &$\mathbf{M}_1$ &$\mathbf{M}_2$ &$\mathbf{M}_3$
        &$\tilde{\mathbf{M}}_2$ &$\tilde{\mathbf{M}}_3$\\ \hline
     sensitivity &0.500 &0.90 &0.10 &0.893 &0.900 &0.977 &0.626 &0.983\\
     specificity &0.500 &0.90 &0.10 &0.812 &0.724 &0.971 &\textbf{0.933} &0.960\\
     precision   &0.375 &0.90 &0.10 &0.806 &0.101 &0.984 &0.377 &0.978\\
     accuracy   &0.500 &0.90 &0.10 &0.850 &0.730 &0.975 &0.915 &0.975\\
     F1-score   &0.428 &0.90 &0.10 &0.847 &0.182 &0.980 &0.470 &0.981\\
     fmi        &0.433 &0.90 &0.10 &0.848 &0.302 &0.981 &0.486 &0.981\\
     auc           &0.500 &0.90 &0.10 &0.853 &0.812 &0.974 &0.779 &0.972\\ 
     $\kappa$   &\textbf{0.000} &0.800 &-0.8 &0.701 &0.129 &0.945 &0.423 &0.946\\
     mcc.s        &0.500 &0.90 &0.10 &0.852 &0.623 &0.973 &0.722 &0.973\\
     nmi    &\textbf{0.000} &0.361 &0.361 &0.249 &0.038 &0.699 &0.113 &0.700\\
     cen.s         &0.028 &0.57 &\textcolor{red}{-0.04}&0.452 &0.580 &0.844 &0.702 &0.845\\   
     mcen.s      &0.123 &0.55 &\textbf{0.00} &0.455 &0.638 &0.826 &0.687 &0.827\\  
     \textbf{eve}&\textbf{0.000} &\textbf{0.99} &\textbf{0.00} 
     		&\textbf{0.976} &\textbf{0.952} &\textbf{0.999} &0.912 
		&\textbf{0.999}\\ \hline
    \end{tabular}
    \end{center}
\end{table}
The Table \ref{tab-twoclass-results} shows evaluation measures values 
for the examples $\mathbf{M}_a$, $\mathbf{M}_b$,
$\mathbf{M}_c$, $\mathbf{M}_1$ and $\mathbf{M}_3$. The proposed 
measure has the better values for almost all examples. 
Undesirable negative value 
is associated with the CEN method for example $\mathbf{M}_c$. The method in 
\cite{Delgado-al-2019} allows to correct the CEN value for a binary problem.
For the balanced and imbalanced matrices $\mathbf{M}_1$ and $\mathbf{M}_2$
associated with a diagnostic test, 
the precision and the F1.score drop heavily ($87\%$ and $78\%$), while the 
accuracy, the AUC and EVE drop slightly ($14\%$, $5\%$ and $2.4\%$). 
The use of the
estimate confusion matrix $\tilde{\mathbf{M}}_2$, leads to an improve of
the precision, the F1-score and the Matthews' correlation coefficient. The
use of $\tilde{\mathbf{M}}_2$ leads to a small change in the EVE
value: $-4.2\%$. There is no significant change in the values when using
$\mathbf{M}_3$ or $\tilde{\mathbf{M}}_3$. The NMI allows to detect the
random classifier example $\mathbf{M}_a$. However, it makes no difference 
between the results for the examples $\mathbf{M}_b$ and $\mathbf{M}_c$ 
which have opposite quality. The NMI, CEN and MCEN have relatively small
values compared to the other measures for these examples.

\subsection{Multi-class examples} 
For illustration purpose, the following CM examples are used.
\begin{equation}
    \mathbf{M}_4 = \left(\begin{array}{ccc}
        50 &0 &0\\  0 &35 &7\\  0 &15 &43
    \end{array}\right) \mbox{ ; }
    \mathbf{M}_5 = \left(\begin{array}{ccc}
        48 &28 &19\\  5 &42 &23\\  14 &9 &44
    \end{array}\right) \nonumber
\end{equation}
\begin{equation}
\mathbf{M}_6 = \left(\begin{array}{ccccc}
17 &2   &0   &0 &0\\
28 &127 &0   &0 &0\\
16 &0   &122 &6 &0\\
6  &0   &4   &\mathbf{3} &0\\
0  &0   &0   &0 &127
\end{array}\right)\mbox{ ; }
\mathbf{M}_7 = \left(\begin{array}{ccccc}
17 &2   &0   &0 &0\\
28 &127 &0   &\mathbf{3} &0\\
16 &0   &122 &6 &0\\
6  &0   &4   &0 &0\\
0  &0   &0   &0 &127
\end{array}\right) \nonumber
\end{equation} 
$\mathbf{M}_4$
is a classifier result for the iris dataset, \cite{Fisher-1936}, 
which has one well separated class and 
two overlapped classes. The matrix $\mathbf{M}_5$ has complex-valued eigenvalues 
and an imbalanced ratio $ir\approx 0.8$. The matrix $\mathbf{M}_6$ corresponds to a 
forest images dataset
cited in \cite{Jupp-1989}, and $\mathbf{M}_7$ is its modification 
(see the entry in bold for class 4). 
The matrices $\mathbf{M}_6$ and $\mathbf{M}_7$ have
the same imbalanced ratio, $ir\approx 0.07$. 
The bounds and eigenvalues of the matrices $\mathbf{A}$ and $\mathbf{B}$ 
associated with these four examples are shown in Table \ref{tab-m4m7-eigen}.

\begin{table}[!htb]
    \caption{Bounds and eigenvalues of $\mathbf{A}$ and $\mathbf{B}$}
        \label{tab-m4m7-eigen}
    \begin{center}
    \begin{tabular}{rcccc}
        &$\mathbf{M}_4$ &$\mathbf{M}_5$ &$\mathbf{M}_6$ &$\mathbf{M}_7+\mathbf{J}_5/5$\\ \hline
    $thr_{min}$  &0.716 &0.279 &0.144 &-3.361\\
    $thr_{max}$  &1.283 &1.721 &1.855 &5.361\\
    $\mu_1$    &1.283 &1.712 &1.765 &3.871\\
    $\mu_2$    &1.000 &0.654 &1.322 &1.179\\
    $\mu_3$    &0.716 &0.633 &1.000 &0.999\\
    $\mu_4$    & -    & -    &0.541 &0.597\\
    $\mu_5$    & -    & -    &0.371 &-1.65\\
    $\lambda_1$&1.014 &1.018 &1.149 &1.150\\
    $\lambda_2$&1.000 &0.411 &1.033 &0.994\\
    $\lambda_3$&0.545 &0.330 &1.000 &0.987\\
    $\lambda_4$& -    & -    &0.185 &0.181\\
    $\lambda_5$& -    & -    &0.171 &-0.104\\ \hline 
    \end{tabular}
\end{center}
\end{table}
In Table \ref{tab-m4m7-eigen}, a modification of $\mathbf{M}_7$ is used
because of the zero diagonal entry for class 4.
The EVE values for $\mathbf{M}_7$ and $\mathbf{M}_7+\mathbf{J}_5/5$ are
$0.77604$ and $0.77539$, respectively.
For the multi-class examples, the eigenvalues bounds are tight for a good 
classifier like for a binary problem.

\subsubsection{Direct use of the CM entries}
The measures in relation (\ref{eq-acc}) - (\ref{eq-mcc}), the normalized
mutual information, the confusion entropy, the modified confusion entropy
measures and that proposed
are used. The results obtained are shown in Table \ref{tab-multiclass-results-a}.

\begin{table}[!htb]
    \caption{Measure results for multi-class problem examples}
        \label{tab-multiclass-results-a}
    \begin{center}
    \begin{tabular}{rccccccc}
     &$\mathbf{M}_4$ &$\mathbf{M}_5$ &$\mathbf{M}_6$ &$\mathbf{M}_7$
     &$\tilde{\mathbf{M}}_5$ &$\tilde{\mathbf{M}}_6$ &$\tilde{\mathbf{M}}_7$\\ \hline
     acc        &0.853 &0.577 &0.865 &0.858 &0.584 &0.818 &0.798\\
     $\kappa$&0.780&0.371 &0.816 &0.806 &0.379 &0.756 &0.730\\
     mcc.s     &0.892 &0.689 &\textbf{0.912} &\textbf{0.908} &0.691 
     		&\textbf{0.887} &\textbf{0.875}\\
    nmi         &0.523 &0.079 &0.629 &0.618 &0.081 &0.592 &0.552\\       
     cen.s      &0.774 &0.354 &0.861 &0.852 &0.352 &0.847 &0.822\\    
    mcen.s      &0.697 &0.237 &0.799 &0.784 &0.231 &0.794 &0.726\\
    \textbf{eve}&\textbf{0.968} &\textbf{0.883} &0.859 &\textcolor{blue}{0.776} 
    		&\textbf{0.889} &0.756 &0.761\\ \hline
    \end{tabular}
    \end{center}
\end{table}
For the balanced matrix $\mathbf{M}_4$ good results are obtained for all measures
in the comparison. Small values are obtained for the matrix $\mathbf{M}_5$, except for
the proposed measure. The matrix $\mathbf{M}_5$
has relatively large off-diagonal entries but it still leads to a
diagonally dominant matrix $\mathbf{P}$. The estimate CM leads
to a small increase measure
values for the matrix $\tilde{\mathbf{M}}_5$ and a small decrease for the matrices
$\tilde{\mathbf{M}}_6$ and $\tilde{\mathbf{M}}_7$. 
The Matthews correlation coefficient is the better for examples 
$\mathbf{M}_6$ and $\mathbf{M}_7$.
Only the proposed measure shows
clearly the modification between $\mathbf{M}_6$ and $\mathbf{M}_7$. 
For these two matrices,
there is a drop of $9.7\%$ for the proposed measure value compared to near $1\%$
for the other measures in the comparison.

\subsubsection{Conversion into a single binary problem}
The entries of a binary CM are obtained from those of multi-class
CM by using the expressions in relations
(\ref{eq-multi2two}) and (\ref{eq-multi2two2}).
\begin{table}[!htb]
    \caption{Multi-class problem converted into a single binary problem}
        \label{tab-multiclass-results-b}
    \begin{center}
    \begin{tabular}{rccccccc}
     &$\mathbf{M}_4$ &$\mathbf{M}_5$ &$\mathbf{M}_6$ &$\mathbf{M}_7$
     &$\tilde{\mathbf{M}}_5$ &$\tilde{\mathbf{M}}_6$ &$\tilde{\mathbf{M}}_7$\\ \hline
     sen	&0.775 &0.433 &0.912 &0.912 &0.435 &0.910 &0.893\\
     spe	&0.881 &\textbf{0.707} &0.918 &0.912 &$\mathbf{0.713}$ &0.878 &0.862\\
     pre	&0.762 &0.426 &0.789 &0.778 &0.431 &0.694 &0.654\\
     acc      &0.846 &0.616 &0.916 &0.912 &0.621 &0.886 &0.869\\
  f1-s		&0.768 &0.429 &0.846 &0.840 &0.433 &0.787 &0.755\\
  fmi   &0.768 &0.429 &0.848 &0.842 &0.433 &0.795 &0.765\\
    auc	&0.828 &0.570 &0.915 &0.912 &0.574 &0.894 &0.878\\
$\kappa$  &0.653 &0.140 &0.789 &0.779 &0.149 &0.712 &0.669\\
     mcc.s   &0.827 &0.570 &0.896 &0.892 &0.574 &0.862 &0.842\\
     nmi     &0.207  &0.008 &0.371 &0.359 &0.009 &0.291 &0.251\\
     cen.s   &0.445  &0.117 &0.645 &0.635 &0.122 &0.584 &0.549\\
    mcen.s   &0.697  &0.237 &0.799 &0.784 &0.183 &0.584 &0.554\\
\textbf{eve}&\textbf{0.966} &0.483 &\textbf{0.994} &\textbf{0.993} 
		&0.501 &\textbf{0.989} &\textbf{0.986}\\ \hline
    \end{tabular}
    \end{center}
\end{table}

Table \ref{tab-multiclass-results-b} shows the results obtained with the 
same evaluation measures used in Table \ref{tab-twoclass-results}. 
The proposed measure has the better value except for the example 
$\mathbf{M}_5$ which has relatively large off-diagonal entries. For the examples
used, when the multi-class example is converted into a single binary problem, 
more large values are obtained for the specificity, the accuracy, the AUC, the 
MCC and the  precision.

\subsubsection{Conversion into many binary problems}
Each class of a multi-class problem is compared to the others taken together
leading to  an analysis of a $2\times 2$ CM for each class. That typically leads to
an imbalanced classification problem. The results obtained are shown in Tables 
\ref{tab-multiclass-results-c4} - \ref{tab-multiclass-results-c7}.
\begin{table}[!htb]
    \caption{Each class versus the others}
        \label{tab-multiclass-results-c4}
    \begin{center}
    \begin{tabular}{lrcccccc}
        & &\multicolumn{3}{c}{non modified CM} &\multicolumn{3}{c}{Modified CM}\\
    $\mathbf{M}_4$  &$j$ &1 &2 &3 &1 &2 &3\\ \hline
    &pre &\textbf{1.0}  &0.833 &0.741 &\textbf{1.0}  &0.876 &0.802\\
    &f1-s  &\textbf{1.0} &0.781 &0.796 &\textbf{1.0}  &0.728 &0.807\\
    &fmi &\textbf{1.0} &0.764 &0.798 &\textbf{1.0}  &0.738 &0.807\\
	&auc &\textbf{1.0} &0.815 &0.855 &\textbf{1.0}  &0.786 &0.851\\
	&acc &\textbf{1.0} &0.853 &0.853 &\textbf{1.0}  &0.830 &0.862\\
&$\kappa$ &\textbf{1.0}  &0.656 &0.682 &\textbf{1.0}  &0.609 &0.699\\
&mcc.s 	&\textbf{1.0}  &0.831 &0.844 &\textbf{1.0}  &0.814 &0.850\\
&nmi    &\textbf{1.0} &0.218 &0.238 &\textbf{1.0}  &0.198 &0.246\\
&cen.s   &\textbf{1.0}  &0.479 &0.470 &\textbf{1.0}  &0.465 &0.742\\
&mcen.s  &\textbf{1.0}  &0.481 &0.473 &\textbf{1.0}  &0.484 &0.472\\
&\textbf{eve}& \textbf{1.0} &\textbf{0.948} &\textbf{0.979}
&\textbf{1.0}  &\textbf{0.913}  & \textbf{0.976}\\  \hline               
    \end{tabular}
    \end{center}
\end{table}

For the example $\mathbf{M}_4$, Table \ref{tab-multiclass-results-c4},
all measures show the well separation of the first class from the other two.
The proposed measure shows better result when using the CM or its estimate.
Using the estimate CM leads to a decrease for the EVE value and an 
improvement for many other measures in the comparison.
 
\begin{table}[!htb]
    \caption{Each class versus the others}
        \label{tab-multiclass-results-c5}
    \begin{center}
    \begin{tabular}{lrcccccc}
        & &\multicolumn{3}{c}{non modified CM} &\multicolumn{3}{c}{Modified CM}\\
$\mathbf{M}_5$  &$j$ &1 &2 &3 &1 &2 &3\\ \hline
    &pre &0.505 &0.600 &0.657 &0.616 &0.676 &\textbf{0.713}\\
    &f1-s&0.592 &0.564 &0.575 &0.616 &0.539 &0.549\\
    &fmi  &0.602 &0.565 &0.579 &0.616 &0.551 &0.564\\
     &auc &0.716 &0.674 &0.677 &0.707 &0.665 &0.660\\
     &acc &0.715 &0.719 &0.720 &0.735 &0.700 &0.698\\
    &$\kappa$&0.384 &0.358 &0.371 &0.414 &0.330 &0.339\\
     &mcc.s &0.699 &0.680 &0.689 &0.707 &0.673 &0.680\\
     &nmi   &0.066 &0.053 &0.059 &0.071 &0.050 &0.055\\
     &cen.s &0.257 &0.237 &0.248 &0.247 &0.249 &0.261\\
     &mcen.s&0.307 &0.275 &0.291 &0.283 &0.302 &0.319\\
     &\textbf{eve} &\textbf{0.883} &\textbf{0.789} &\textbf{0.781}
     &\textbf{0.860} &\textbf{0.709} &\textbf{0.713}\\  \hline                
    \end{tabular}
    \end{center}
\end{table}

For the example $\mathbf{M}_5$, Table \ref{tab-multiclass-results-c5}, like for the
example $\mathbf{M}_4$, the proposed measure has the better results when using CM
or its estimate. Again, there is an improvement for some measures when using
the estimate CM but not for EVE which value decreases.

\begin{table}[!htb]
    \caption{Each class versus the others}
        \label{tab-multiclass-results-c6}
    \begin{center}
    \begin{tabular}{lrcccccccccc}
        & &\multicolumn{5}{c}{non modified CM} &\multicolumn{5}{c}{Modified CM}\\
    $\mathbf{M}_6$  &$j$ &1 &2 &3 &4 &5 &1 &2 &3 &4 &5\\ \hline
   	&pre &0.895 &0.819 &0.847 &0.231 &\textbf{1.0} 
        &0.953 &0.878 &0.900 &0.679 &\textbf{1.0}\\
    &f1-s&0.395 &0.894 &0.904 &0.272 &\textbf{1.0} 
        &0.218 &0.924 &0.924 &0.120 &\textbf{1.0}\\
    &fmi &0.476 &0.898 &0.906 &0.277 &\textbf{1.0} 
        &0.343 &0.926 &0.924 &0.212 &\textbf{1.0}\\
	&auc &0.624 &0.949 &0.951 &0.655 &\textbf{1.0} 
        &0.560 &0.960 &0.953 &0.531 &\textbf{1.0}\\
	&acc &0.886 &0.934 &0.943 &0.965 &\textbf{1.0} 
        &0.769 &0.953 &0.956 &0.909 &\textbf{1.0}\\
    &$\kappa$&0.353 &0.847 &0.864 &0.255 &\textbf{1.0} 
        &0.169 &0.891 &0.893 &0.106 &\textbf{1.0}\\
	&mcc &0.720 &0.927 &0.933 &0.630 &\textbf{1.0} 
        &0.647 &0.947 &0.947 &0.596 &\textbf{1.0}\\
    &nmi &\textbf{0.118} &0.494 &0.506 &\textbf{0.056} &\textbf{1.0} 
        &0.057 &0.563 &0.552 &0.027 &\textbf{1.0}\\
&cen.s   &0.705 &0.727 &0.737 &0.857 &\textbf{1.0} 
    &0.613 &0.772 &0.766 &0.775 &\textbf{1.0}\\
&mcen.s  &0.711 &0.719 &0.724 &0.839 &\textbf{1.0} 
    &0.661 &0.756 &0.747 &0.768 &\textbf{1.0}\\
 &\textbf{eve} &0.393 &\textbf{0.997} &\textbf{0.998} &0.585 &\textbf{1.0}
 &\textbf{0.000} &\textbf{0.998} &\textbf{0.998} &\textbf{0.000} &\textbf{1.0}\\  \hline                
    \end{tabular}
    \end{center}
\end{table}
\begin{table}[!htb]
    \caption{Each class versus the others}
        \label{tab-multiclass-results-c7}
    \begin{center}
    \begin{tabular}{lrcccccccccc}
        & &\multicolumn{5}{c}{non modified CM} &\multicolumn{5}{c}{Modified CM}\\
    $\mathbf{M}_7$  &$j$ &1 &2 &3 &4 &5 &1 &2 &3 &4 &5\\ \hline
   	&pre &0.895 &0.804 &0.847 &\textbf{0.000} &\textbf{1.0} 
        &0.953 &0.867 &0.900 &0.242 &\textbf{1.0}\\
    &f1s&0.395 &0.885 &0.904 &\textcolor{red}{NA} &\textbf{1.0} 
        &0.218 &0.918 &0.924 &0.015 &\textbf{1.0}\\
    &fmi &0.476 &0.889 &0.906 &0.000 &\textbf{1.0} 
        &0.343 &0.919 &0.924 &0.043 &\textbf{1.0}\\
	&auc &0.624 &0.945 &0.951 &0.489 &\textbf{1.0} 
        &0.560 &0.957 &0.953 &0.502 &\textbf{1.0}\\
	&acc &0.886 &0.928 &0.943 &0.958 &\textbf{1.0} 
        &0.769 &0.949 &0.956 &0.871 &\textbf{1.0}\\
&$\kappa$&0.353 &0.833 &0.864 &\textcolor{red}{-0.02} &\textbf{1.0} 
        &0.169 &0.882 &0.893 &0.007 &\textbf{1.0}\\
	&mcc.s &0.720 &0.920 &0.933 &0.489 &\textbf{1.0} 
        &0.647 &0.942 &0.947 &0.511 &\textbf{1.0}\\
    &nmi &\textbf{0.118} &0.472 &0.506 &0.002 &\textbf{1.0} 
        &0.057 &0.544 &0.552 &\textbf{0.000} &\textbf{1.0}\\
&cen.s   &0.705 &0.711 &0.737 &0.843 &\textbf{1.0} 
    &0.613 &0.758 &0.766 &0.728 &\textbf{1.0}\\
&mcen.s  &0.711 &0.706 &0.724 &0.820 &\textbf{1.0} 
        &0.661 &0.744 &0.747 &0.728 &\textbf{1.0}\\
 &\textbf{eve} &0.393 &\textbf{0.996} &\textbf{0.998} 
 		&\textbf{0.000} &\textbf{1.0} 
    &\textbf{0.000} &\textbf{0.998} &\textbf{0.998} &\textbf{0.000} &\textbf{1.0}\\  \hline                
    \end{tabular}
    \end{center}
\end{table}

For the examples $\mathbf{M}_6$ and $\mathbf{M}_7$, 
Tables \ref{tab-multiclass-results-c6} and \ref{tab-multiclass-results-c7}, 
respectively, all measures show the well separation of the fifth class 
compared to the others. In the example $\mathbf{M}_6$, 
more than half of the observations in classes 1 and 4 are wrongly predicted. 
The diagonal entry for class 4 is empty for the example 
$\mathbf{M}_7$. When using the estimate CM, there is an improvement for
all measures. The measures values are relatively small for classes 1 and 4.
The change in class 4 which transform $\mathbf{M}_6$ into $\mathbf{M}_7$
is detected by all measures except CEN and MCEN which values vary slightly.
Unlike for the examples $\mathbf{M}_4$ and $\mathbf{M}_5$, the use of the estimate
CM does not lead to a decrease of the EVE value.

\subsubsection{MNIST dataset}
The MNIST dataset consists of images for handwritten digits, 0, 1, \ldots 10.
There are a training set of 60,000 images and test set of 10,000 images. Each 
observation is a grayscale image of size 28 x 28 leading to a vector of size 
784 which values are from the set $\{0, 1, \ldots, 256\}$. 
Labels information are available for all images. Here, the linear discriminant 
algorithm (LDA) is applied to the training set to obtain a model 
parameter matrix of size $784\times 10$. This model is used with the test set 
to obtain a prediction assignment matrix of size $10,000\times 10$. The entries
of each row (vector) of this matrix give proximity to the $10$ digits. 
The maximum vector entry of the LDA model is used to form the confusion matrix 
$\mathbf{M}_8$.
To obtain the confusion matrix $\mathbf{M}_9$, the LDA model vector entries 
for each test image are normalized to have sum one, membership values.
Then, these membership values are added for each class to obtain the matrix 
$\mathbf{M}_9$.
\[\mathbf{M}_8 = \left(\begin{array}{cccccccccc}
    939 &0   &17  &5   &1   &23  &17  &5   &14  &17\\
    1   &1106&58  &17  &21  &16  &9   &40  &52  &11\\
    2   &2   &820 &21  &6   &4   &11  &17  &9   &4\\
    3   &2   &26  &887 &1   &90  &0   &8   &30  &16\\
    1   &1   &16  &2   &873 &17  &19  &19  &32  &66\\
    11  &1   &0   &12  &5   &616 &18  &1   &44  &0\\
    13  &5   &34  &12  &10  &23  &874 &2   &17  &1\\
    1   &2   &17  &24  &1   &15  &0   &872 &11  &74\\
    8   &16  &39  &18  &12  &62  &10  &5   &741 &12\\
    1   &0   &5   &12  &52  &26  &0   &59  &24  &808            
\end{array}\right)\]

\[\mathbf{M}_9 = \left(\begin{array}{cccccccccc}
    358.72 &28.22 &77.91 &78.47 &43.83 &87.26 &76.82 &61.75 &65.44 &48.65\\
    39.85 &559.51 &91.05 &87.34 &60.11 &58.08 &62.44 &62.06 &90.47 &60.28\\
    73.24 &82.13 &304.03 &92.69 &59.87 &49.34 &88.77 &81.59 &75.25 &53.07\\
    76.46 &89.64 &106.79 &302.24 &57.14 &99.99 &51.23 &81.44 &81.89 &65.30\\
    49.64 &45.77 &64.79 &55.22 &318.28 &66.97 &86.53 &64.93 &81.98 &141.22\\
    90.80 &47.62 &55.54 &109.30 &70.59 &231.95 &73.20 &58.43 &108.59 &64.39\\
    86.12 &57.39 &92.38 &55.86  &82.26 &67.64 &336.46 &47.78 &69.91 &61.94\\
    67.03 &58.02 &75.59 &78.50 &67.47 &58.59 &50.64 &366.61 &51.59 &130.04\\
    78.81 &104.32 &98.75 &83.29 &85.08 &111.45 &66.51 &59.52 &272.41 &83.02\\
59.33 &62.38 &65.16 &67.07 &137.37 &60.72 &65.39 &143.88 &76.47 &301.09  
\end{array}\right)\]

The eigenvalues and their bounds are shown in Table \ref{tab-multiclass-mnist}.
Both $\mathbf{M}_8$ and $\mathbf{M}_9$ lead to a positive
define stochastic matrix 
$\mathbf{P}$. However, the eigenvalues bounds for $\mathbf{M}_8$ are tighter
than those of $\mathbf{M}_9$.
\begin{table}[!htb]
    \caption{MNIST results}
        \label{tab-multiclass-mnist}
    \begin{center}
    \begin{tabular}{rcc}  \hline   
    $\mathbf{M}_8$ &$thr_{min}$=0.731 
        &$\sigma(\mathbf{A})=\{$1.198, 1.074, 1.025, 1.013, 0.999, 0.985, 0.966, 0.946, 0.899, 0.893$\}$\\
    &$thr_{max}$=1.269 
        &$\sigma(\mathbf{B})=\{$1.019, 0.965, 0.933, 0.908, 0.893, 0.852, 0.785, 0.763, 0.738, 0.649$\}$\\ \hline 
    $\mathbf{M}_9$ &$thr_{min}$=-1.469 
        &$\sigma(\mathbf{A})=\{$3.119, 1.115, 0.967, 0.913, 0.884, 0.738, 0.672, 0.652, 0.488, 0.451$\}$\\
    &$thr_{max}$=3.469  
        &$\sigma(\mathbf{B})=\{$1.001, 0.439, 0.359, 0.311, 0.273, 0.255, 0.213, 0.194, 0.150, 0.128$\}$\\ \hline
    \end{tabular}
    \end{center}
\end{table}

The measures based on a direct use of the matrix $\mathbf{P}$ entries are shown in
Table \ref{tab-multiclass-mnist2}. The EVE measure has again the better value.
When using the real-valued (soft assignment) CM, the EVE value drops by $8.4\%$
while the other measures, in the comparison drop by $31\%$ to $90\%$.
\begin{table}[!htb]
    \caption{MNIST results}
        \label{tab-multiclass-mnist2}
    \begin{center}
    \begin{tabular}{rcccccc} 
        &acc    &$\kappa$ &mcc.s &nmi &cen.s &\textbf{eve}\\  \hline   
        $\mathbf{M}_8$ &0.854 &0.837 &0.919 &0.555 &0.784&\textbf{0.996}\\ 
        $\mathbf{M}_9$ &0.335 &0.261 &0.630 &0.054 &0.251&\textbf{0.912}\\   
        drop     &$60.8\%$ &$68.8\%$&$31.4\%$&$90.3\%$&$67.9\%$ &$8.4\%$\\  \hline   
    \end{tabular}
    \end{center}
\end{table}

The imbalanced ratio for the examples $\mathbf{M}_8$ and $\mathbf{M}_9$
is $ir=0.786$ and the use of the estimate confusion matrices 
$\tilde{\mathbf{M}}_8$ and $\tilde{\mathbf{M}}_9$ does not lead to
a significant difference compare to the results reported in
Table \ref{tab-multiclass-mnist2}.
For reproducibility of the results presented a publicly R package,
eve, is available.

\section{Conclusion}\label{sect-conclusion}
While many evaluation measures consist of a direct use of the confusion 
matrix entries,
the proposed method necessitates extra computation to get a matrix 
$\mathbf{P}$ and the eigenvalues of a symmetric matrix $\mathbf{B}$.
The eigenvalue decomposition is the heavy computational part of the proposed 
measure.  From the results of Abel and Galois, there is no explicit expressions
for the eigenvalues when $n>4$. However, many tools are available
for computing the eigenvalues even for a large $n$ value 
\cite{Golub-al-1996,Bjorck-2015}. 
Explicit relations are provided for the eigenvalue of a binary 
classification problem.
A drawback of the EVE measure for a binary problem is it makes no difference
between a random assignment for the two classes and the results where one 
class is wrongly predicted and the other is correctly predicted, presence of
negative eigenvalue. Since it is shown that the sum of the eigenvalues for
a binary problem is related to the AUC and the Gini index, a negative 
eigenvalue will also lead to a small value for these two measures.

The proposed evaluation measure is based on a linear algebra 
analysis of the confusion matrix (CM) unlike other measures 
based on the statistical behavior of the classes associated with 
the CM. However, for a binary problem, a link has been shown 
between the eigenvalues used and the 
specificity, the sensitivity, the AUC and the Gini index. 
The comparison results presented show a better performance 
for the proposed measure for balanced and imbalanced examples. 
An expression allowing to compute lower and upper bounds for the 
eigenvalues is provided as a function of the CM entries. The 
examples results show that tighter bounds are associated with 
a good classifier. An extension of this work can be a
statistical characterization of the eigenvalues.

An expression allowing to estimate the CM from that observed 
is also provided. 
This estimate CM leads to an improvement for some commonly 
used measures when the imbalanced ratio is small, $<0.1$ for 
the tested examples. Here,
classification problem examples are used for the illustration. 
The proposed measure also apply to two clustering methods 
results or for comparing two raters results, where the number $k$ 
of clusters/categories is the same.
The EVE measure may be a good candidate for evaluating a
classification method or for comparing two clustering methods,
evaluating the agreement of two raters.

\subsubsection*{Acknowledgements.} 
The work of D. Demb\'el\'e is supported by CNRS.

\subsubsection*{Funding resources.}
This research did not receive any specific grant from funding agencies
in public, commercial or not-a-profit sectors.

\bibliographystyle{abbrv}
\bibliography{/Users/doulayedembele/Documents/bib/biblio_af,
/Users/doulayedembele/Documents/bib/biblio_go,
/Users/doulayedembele/Documents/bib/biblio_ps,
/Users/doulayedembele/Documents/bib/biblio_tz}

\end{document}